\documentclass[12pt]{amsart}

\usepackage[utf8]{inputenc}

\usepackage[letterpaper, margin=1in]{geometry}

\usepackage{amsmath, amssymb, amsthm}
\usepackage{mathrsfs}

\usepackage{algorithm}
\usepackage{algpseudocode}
\usepackage{float}

\usepackage{graphicx}
\usepackage{pgfplots}
\pgfplotsset{compat=1.18}

\usepackage{microtype}
\usepackage{xurl}
\usepackage{cite}
\usepackage[colorlinks=true, linkcolor=blue, citecolor=blue, urlcolor=blue]{hyperref}

\newtheorem{theorem}{Theorem}[section]
\newtheorem{lemma}[theorem]{Lemma}
\newtheorem{proposition}[theorem]{Proposition}
\newtheorem{corollary}[theorem]{Corollary}

\theoremstyle{definition}

\newtheorem{assumption}[theorem]{Assumption}

\theoremstyle{remark}
\newtheorem{remark}[theorem]{Remark}

\numberwithin{equation}{section}

\title[Transfer Operators for Change Detection]{Empirical Transfer Operators and Finite-Sample Change Detection for Noisy Expanding Interval Maps}

\author{APARNA RAJPUT}
\date{June 2026}
\address{Department of Mathematics and Statistics, Concordia University, 1455 De Maisonneuve Blvd.\\
Montreal, QC H3G 1M8, CANADA\\
aparna.ar.rajput@gmail.com\\
ORCID: 0009-0002-9932-1586}

\begin{document}
\title{
Transfer-Operator-Based Change Detection with Finite-Sample Guarantees
}

\begin{abstract}
We study a finite-sample change-detection problem for one-dimensional noisy dynamical
systems using partition-based empirical approximations of stationary behaviour. Given
observations from an interval-valued process, we partition the state space into finitely many
intervals and estimate a transition matrix from observed transitions between partition
elements. After a small Doeblin-type regularisation, the resulting matrix has a unique
stationary distribution. This stationary distribution is used as a finite-dimensional
approximation of the invariant density, or stationary law, of the observed regime.

Using an initial reference segment, we compute a baseline empirical stationary distribution
\(\widehat{\pi}_{0,\rho}\). For each subsequent sliding window, we compute a window-based
empirical stationary distribution \(\widehat{\pi}_{t,\rho}\) and define the score
\[
S_t = \|\widehat{\pi}_{t,\rho} - \widehat{\pi}_{0,\rho}\|_1.
\]
Large values of \(S_t\) indicate that the stationary behaviour of the observed regime has
changed relative to the baseline. The statistic is therefore a detector of changes in
stationary behaviour. It is not, by itself, a detector of all possible changes in transition
dynamics that preserve the invariant density.

Under explicit assumptions on concentration of empirical transition probabilities, stability
of stationary distributions for the regularised finite-state model, approximation of the
observed invariant density by the chosen partition, and stability under observational noise,
we derive a finite-sample bound for the empirical stationary density. The bound separates
sampling error, regularisation bias, partition approximation error, and noise bias. We then
obtain a single-window threshold rule that controls the false-alarm probability under the
no-change hypothesis and gives a sufficient condition for detection when the invariant
density changes by more than the estimation error.

We illustrate the method on synthetic examples and specify the experiments needed for a
complete empirical validation of a change detector.
\end{abstract}

\maketitle

\section{Introduction}

Detecting sustained changes in the dynamics of a time series is a fundamental problem in applications such as signal processing, control, energy monitoring, finance, and predictive maintenance. 
Many classical anomaly-detection methods focus on pointwise outliers or changes in low-order summary statistics such as the mean and variance. 
Such methods can be effective for detecting isolated abnormal observations, but they may fail to detect changes in the underlying evolution law when the marginal observations do not change dramatically.

In this work, we study change detection from an operator-theoretic viewpoint. 
Suppose that an observed process takes values in the interval \(I=[0,1]\) and evolves according to an underlying dynamical mechanism. 
For a deterministic map \(T:I\to I\), the associated Frobenius--Perron operator describes how probability densities are transported under the dynamics. 
If the system admits an absolutely continuous invariant measure with density \(h\), then \(h\) represents the long-term statistical behavior of the system.

A change in the dynamics can therefore be detected by comparing invariant or stationary densities across different time windows. 
Instead of asking whether a single observation is unusual, we ask whether the long-term statistical behavior of the system has changed.

The proposed method proceeds as follows. 
We partition the state space \(I=[0,1]\) into finitely many intervals and convert the observed process into a finite-state process by recording which interval contains each observation. 
From observed transitions between intervals, we estimate a finite transition matrix. 
This construction is inspired by Ulam-type finite-state approximations of transfer
operators, but in the present statistical setting the matrix is estimated from observed
transitions of the binned process.
The stationary distribution of the empirical transition matrix is then used as a finite-dimensional approximation of the invariant density.

Given an initial baseline segment, we compute a baseline empirical stationary distribution \(\widehat h_0\). 
For each subsequent sliding window, we compute a window-based empirical stationary distribution \(\widehat h_t\). 
We define the anomaly score
\[
    S_t = \|\widehat h_t-\widehat h_0\|_1 .
\]
A large value of \(S_t\) indicates that the stationary behavior in the current window differs significantly from the baseline behavior.

The main purpose of this paper is to justify this score mathematically under explicit assumptions. 
We separate the error in the empirical invariant distribution into four components:
\[
    \text{sampling error}
    +
    \text{regularization bias}
    +
    \text{partition approximation error}
    +
    \text{noise or perturbation bias}.
\]
More precisely, under assumptions on concentration of empirical transition probabilities,
positive cell masses, partition approximation of the observed invariant density, regularisation
bias, and stability of invariant densities under noise, we obtain a bound of the form
\[
\|E_N\widehat\pi_{n,\rho}-h\|_{L^1}
\le
\frac{N(1-\rho)}{\rho}\varepsilon
+\zeta_N(\rho)
+C_A\Delta_N^\beta
+\eta(\delta)
\]
with high probability. Here \(\varepsilon\) represents the empirical transition-matrix
sampling error, \(\zeta_N(\rho)\) is the bias introduced by regularisation, \(\Delta_N\)
is the mesh size of the partition, and \(\eta(\delta)\) is a noise-stability term satisfying
\(\eta(\delta)\to0\) as \(\delta\to0\).

This bound leads to a threshold rule for change detection. 
Under the no-change hypothesis, the baseline and window invariant densities agree, and the score \(S_t\) remains small with high probability. 
Under an alternative hypothesis, if the true invariant density in a sliding window differs from the baseline density by more than the estimation error, then \(S_t\) exceeds the threshold with high probability.

The contributions of this work are as follows.

\begin{enumerate}
    \item We formulate a change-detection method based on empirical approximations of transfer operators and invariant densities.

    \item We define an anomaly score using the \(L^1\)-distance between empirical stationary distributions computed from a baseline segment and sliding windows.

    \item Under explicit assumptions, we derive finite-sample bounds for the empirical invariant distribution. 
    The bound separates sampling error, regularization bias, partition approximation error, and noise bias.

    \item We derive a threshold rule that controls the false-alarm probability under the no-change hypothesis and gives a sufficient condition for detection under a genuine change in invariant density.

    \item We provide numerical validation on synthetic noisy beta-map change-point experiments.
\end{enumerate}

The rest of the paper is organized as follows. Section~2 introduces the finite-state
approximation and empirical transition matrices. Section~3 proves the finite-sample
approximation bound for empirical stationary distributions. Section~4 derives single-window
change-detection guarantees and a multiple-window false-alarm corollary. Section~5 explains
the operator-theoretic interpretation and clarifies what is, and is not, detected by the
proposed statistic. Section~6 presents the algorithmic implementation. Section~7 gives the
experimental validation protocol. Section~8 concludes the paper. Such estimates are standard consequences of stochastic stability results for expanding
systems under suitable Lasota--Yorke and spectral-gap assumptions; see, for example,
\cite{lasota_yorke_1973,keller_1985,keller_liverani_1999,rajput2026stochastic}.

\section{Statistical and Operator-Theoretic Preliminaries}
\label{sec:preliminaries}

Let \(I=[0,1]\) and let \((X_t)_{t\geq 0}\) be an observed time series taking values in \(I\).
We interpret \(X_t\) as observations generated by an underlying one-dimensional dynamical system, possibly subject to noise.

In the deterministic case, if \(T:I\to I\) is a nonsingular map, the associated Frobenius--Perron operator \(P_T\) acts on densities and describes how probability distributions are transported by the dynamics.
An invariant density \(h\) satisfies
\[
    P_T h = h.
\]
Thus \(h\) represents the long-term statistical behavior of the system.

In the noisy case, the transfer operator may be replaced by a perturbed operator, denoted \(P_\delta\), where \(\delta\geq 0\) measures the noise level.
For example, one may consider a Markov smoothing perturbation of the form
\[
    P_\delta = Q_\delta P_T,
\]
where \(Q_\delta\) is a Markov kernel representing noise.
We assume that the noisy system admits a unique invariant density \(h_\delta\), and that the family is stable in the sense that
\[
    \|h_\delta-h\|_{L^1} \leq \eta(\delta),
    \qquad \eta(\delta)\to 0
    \quad \text{as } \delta\to 0.
\]
Such estimates are standard consequences of stochastic stability results for expanding systems under suitable Lasota--Yorke and spectral-gap assumptions; see, for example, \cite{lasota_yorke_1973,keller_1985,keller_liverani_1999}.

\subsection{Finite partition and induced process}

Fix a finite partition
\[
    \mathcal A_N=\{I_1,\ldots,I_N\}
\]
of \(I=[0,1]\) into intervals.
Let
\[
    \Delta_N = \max_{1\leq i\leq N} |I_i|
\]
denote the mesh size of the partition.

The observed process \((X_t)\) induces a finite-state process \((Y_t)\) by
\[
    Y_t = i
    \quad \Longleftrightarrow \quad
    X_t \in I_i .
\]
Thus \(Y_t\in\{1,\ldots,N\}\).
For the finite-sample analysis, we assume that under a fixed stationary regime the binned process \((Y_t)\) can be modeled as a stationary, irreducible, aperiodic, geometrically ergodic Markov chain with transition matrix \(P_N\) and stationary distribution \(\pi_N\).

This is an assumption on the finite-dimensional model induced by the partition. It is not automatic for an arbitrary time series or arbitrary partition.
We write
\[
    \pi_{\min,N} = \min_{1\leq i\leq N} \pi_N(i).
\]
Throughout the finite-sample analysis, we assume
\[
    \pi_{\min,N}>0.
\]

\begin{remark}
The geometric ergodicity and positivity assumptions are not automatic consequences of discretization.
They must either be assumed or enforced by regularization.
In applications, some partition cells may be rarely visited, which can cause instability in empirical transition estimates.
\end{remark}

\subsection{Empirical transition matrix}

Given observations
\[
    X_0,X_1,\ldots,X_{n-1},
\]
we obtain the corresponding discretized observations
\[
    Y_0,Y_1,\ldots,Y_{n-1}.
\]

For \(1\leq i,j\leq N\), define the empirical transition counts
\[
    C_{ij}^{(n)}
    =
    \sum_{t=0}^{n-2}
    \mathbf 1_{\{Y_t=i,\;Y_{t+1}=j\}}.
\]
Define also the number of times state \(i\) is available as a departure state:
\[
    C_i^{(n)}
    =
    \sum_{t=0}^{n-2}
    \mathbf 1_{\{Y_t=i\}}.
\]
Notice that the denominator uses \(t=0,\ldots,n-2\), not \(t=0,\ldots,n-1\), because only the first \(n-1\) observations can serve as starting points for observed transitions.

The empirical transition matrix \(\widehat P_n\) is defined by
\[
    \widehat P_n(i,j)
    =
    \begin{cases}
        \dfrac{C_{ij}^{(n)}}{C_i^{(n)}},
        & C_i^{(n)}>0, \\[1.2em]
        \dfrac{1}{N},
        & C_i^{(n)}=0.
    \end{cases}
\]
With this convention, every row of \(\widehat P_n\) sums to one, and therefore \(\widehat P_n\) is a row-stochastic matrix.

\begin{remark}
The uniform row assigned when \(C_i^{(n)}=0\) is a practical convention.
Other choices are possible, but some convention is necessary because the empirical transition probabilities are undefined when a state is never observed as a departure state.
\end{remark}

\subsection{Regularized empirical transition matrix}

A row-stochastic matrix need not have a unique stationary distribution.
For example, it may contain several closed communicating classes.
Therefore, uniqueness of the stationary distribution of \(\widehat P_n\) cannot be asserted from the Perron--Frobenius theorem alone.

To guarantee uniqueness, we use a small regularization parameter \(\rho\in(0,1)\) and define
\[
    \widehat P_{n,\rho}
    =
    (1-\rho)\widehat P_n
    +
    \rho \frac{\mathbf 1\mathbf 1^\top}{N},
\]
where \(\mathbf 1\in\mathbb R^N\) denotes the vector of all ones.

The matrix \(\widehat P_{n,\rho}\) is strictly positive:
\[
    \widehat P_{n,\rho}(i,j)\geq \frac{\rho}{N}>0
    \qquad
    \text{for all } i,j.
\]
Hence it is irreducible and aperiodic.
By the Perron--Frobenius theorem, it has a unique stationary distribution
\[
    \widehat h_{n,\rho}\in\mathbb R^N
\]
satisfying
\[
    \widehat h_{n,\rho}^{\top}\widehat P_{n,\rho}
    =
    \widehat h_{n,\rho}^{\top},
    \qquad
    \widehat h_{n,\rho}\geq 0,
    \qquad
    \|\widehat h_{n,\rho}\|_1=1.
\]

When \(\rho\) is fixed and small, \(\widehat h_{n,\rho}\) is the empirical stationary distribution used in the detection algorithm.

\subsection{From vectors to densities}

The vector \(\widehat h_{n,\rho}\in\mathbb R^N\) is a discrete probability distribution on the partition cells.
To compare it with a density on \(I=[0,1]\), we identify it with a piecewise-constant density.

Define the embedding operator
\[
    \mathcal E_N:\mathbb R^N\to L^1(I)
\]
by
\[
    \mathcal E_N v(x)
    =
    \sum_{i=1}^N
    \frac{v_i}{m(I_i)}
    \mathbf 1_{I_i}(x),
\]
where \(m(I_i)\) is the Lebesgue measure of the interval \(I_i\).

If \(v\geq 0\) and \(\|v\|_1=1\), then \(\mathcal E_N v\) is a probability density on \(I\), since
\[
    \int_I \mathcal E_N v(x)\,dx
    =
    \sum_{i=1}^N
    \frac{v_i}{m(I_i)}m(I_i)
    =
    \sum_{i=1}^N v_i
    =
    1.
\]

Thus, when comparing a discrete stationary distribution with a true invariant density, we use
\[
    \|\mathcal E_N\widehat h_{n,\rho}-h\|_{L^1}.
\]

For notational simplicity, later sections may write \(\widehat h_n\) instead of \(\widehat h_{n,\rho}\), but all empirical stationary distributions are understood to be computed from the regularized matrix \(\widehat P_{n,\rho}\), unless stated otherwise.
\section{Finite-Sample Approximation of the Invariant Density}
\label{sec:finite_sample}

In this section we derive a finite-sample bound for the empirical invariant distribution obtained from the regularized empirical transition matrix.

Let \(P_N\) denote the true transition matrix of the induced finite-state process \((Y_t)\) on the partition \(\mathcal A_N\). 
Let \(\pi_N\) be its unique stationary distribution:
\[
    \pi_N^\top P_N = \pi_N^\top,
    \qquad
    \pi_N\geq 0,
    \qquad
    \|\pi_N\|_1=1.
\]

For \(\rho\in(0,1)\), define the regularized true transition matrix
\[
    P_{N,\rho}
    =
    (1-\rho)P_N
    +
    \rho J_N,
    \qquad
    J_N = \frac{\mathbf 1\mathbf 1^\top}{N}.
\]
Similarly, the regularized empirical transition matrix is
\[
    \widehat P_{n,\rho}
    =
    (1-\rho)\widehat P_n
    +
    \rho J_N.
\]
Both \(P_{N,\rho}\) and \(\widehat P_{n,\rho}\) are strictly positive stochastic matrices. 
Hence each has a unique stationary distribution. 
We denote them by
\[
    \pi_{N,\rho}
    \quad\text{and}\quad
    \widehat \pi_{n,\rho},
\]
respectively.

The empirical invariant density is the piecewise-constant density
\[
    \mathcal E_N \widehat \pi_{n,\rho},
\]
where the embedding operator \(\mathcal E_N\) was defined in Section~\ref{sec:preliminaries}.
We identify the stationary vector \(\pi_N\) with the piecewise-constant density
\[
    h_N = \mathcal E_N\pi_N.
\]
Thus \(h_N\) is the finite-dimensional stationary density associated with the ideal
finite-state transition matrix \(P_N\) of the binned process.
\subsection{Assumptions}

We state the assumptions used in the finite-sample analysis.

\begin{assumption}[Finite-state concentration]
\label{ass:concentration}
Under a fixed stationary regime, the induced finite-state process \((Y_t)\) is geometrically ergodic with transition matrix \(P_N\) and stationary distribution \(\pi_N\). 
Moreover, there exist constants \(C_{\mathrm{conc}},c_{\mathrm{conc}}>0\) such that empirical averages of bounded functions of \((Y_t,Y_{t+1})\) satisfy exponential concentration inequalities of Markov-chain type.
\end{assumption}

\begin{assumption}[Positive cell mass]
\label{ass:pimin}
There exists \(\pi_{\min,N}>0\) such that
\[
    \pi_N(i)\geq \pi_{\min,N}
    \qquad
    \text{for all } i=1,\ldots,N.
\]
\end{assumption}

\begin{assumption}[Noise stability]
\label{ass:noise}
Let \(h\) be the invariant density of the unperturbed system and \(h_\delta\) the invariant density of the noisy system. 
There exists a modulus \(\eta(\delta)\to0\) as \(\delta\to0\) such that
\[
    \|h_\delta-h\|_{L^1}
    \leq
    \eta(\delta).
\]
\end{assumption}

\begin{assumption}[Regularization bias]
\label{ass:regularization_bias}
The regularization of the finite transition matrix introduces a bias
\[
    \zeta_N(\rho)
    =
    \|\mathcal E_N\pi_{N,\rho}-h_N\|_{L^1}.
\]
We assume that \(\zeta_N(\rho)\to0\) as \(\rho\to0\), for fixed \(N\), whenever \(P_N\) has a unique stationary distribution.
\end{assumption}
\begin{remark}
If \(P_N\) has a unique stationary distribution \(\pi_N\), then \(P_{N,\rho}\to P_N\) as \(\rho\to0\). 
Under finite-dimensional stationary-distribution perturbation theory, this implies
\[
    \pi_{N,\rho}\to \pi_N.
\]
Since \(h_N=\mathcal E_N\pi_N\), it follows that
\[
    \zeta_N(\rho)
    =
    \|\mathcal E_N\pi_{N,\rho}-\mathcal E_N\pi_N\|_{L^1}
    \to 0
\]
as \(\rho\to0\).
\end{remark}
\subsection{Concentration of the empirical transition matrix}

We first show that the empirical transition matrix concentrates around the true finite transition matrix.

\begin{lemma}
\label{lem:transition_concentration}
Suppose Assumptions~\ref{ass:concentration} and~\ref{ass:pimin} hold.
Then there exist constants \(C_0,c_0>0\) such that, for all \(0<\varepsilon<1\),
\[
    \mathbb P
    \left(
        \|\widehat P_n-P_N\|_{\max}>\varepsilon
    \right)
    \leq
    C_0 N^2
    \exp\left(
        -c_0 n \pi_{\min,N}^2 \varepsilon^2
    \right).
\]
Consequently,
\[
    \mathbb P
    \left(
        \|\widehat P_{n,\rho}-P_{N,\rho}\|_{\max}>(1-\rho)\varepsilon
    \right)
    \leq
    C_0 N^2
    \exp\left(
        -c_0 n \pi_{\min,N}^2 \varepsilon^2
    \right).
\]
\end{lemma}

\begin{proof}
For \(1\leq i,j\leq N\), define
\[
    C_{ij}^{(n)}
    =
    \sum_{t=0}^{n-2}
    \mathbf 1_{\{Y_t=i,\;Y_{t+1}=j\}},
    \qquad
    C_i^{(n)}
    =
    \sum_{t=0}^{n-2}
    \mathbf 1_{\{Y_t=i\}}.
\]
Then
\[
    \widehat P_n(i,j)
    =
    \frac{C_{ij}^{(n)}}{C_i^{(n)}}
\]
whenever \(C_i^{(n)}>0\).

Under stationarity,
\[
    \mathbb E\left[\frac{C_{ij}^{(n)}}{n-1}\right]
    =
    \pi_N(i)P_N(i,j),
    \qquad
    \mathbb E\left[\frac{C_i^{(n)}}{n-1}\right]
    =
    \pi_N(i).
\]
By Assumption~\ref{ass:concentration}, for each \(i,j\), empirical transition frequencies and empirical state frequencies concentrate around their expectations. 
Thus, for constants \(C,c>0\), with probability at least
\[
1 - CN^2\exp(-cn\alpha^2),
\]
we have simultaneously
\[
    \left|
    \frac{C_{ij}^{(n)}}{n-1}
    -
    \pi_N(i)P_N(i,j)
    \right|
    \leq \alpha
\]
for all \(i,j\), and
\[
    \left|
    \frac{C_i^{(n)}}{n-1}
    -
    \pi_N(i)
    \right|
    \leq \alpha
\]
for all \(i\).

Choose
\[
    \alpha = \frac{\pi_{\min,N}\varepsilon}{4}.
\]
For \(0<\varepsilon<1\), this choice satisfies
\[
    \alpha \leq \frac{\pi_{\min,N}}{4}
    < \frac{\pi_{\min,N}}{2}.
\]
Then, on the above event,
\[
    \frac{C_i^{(n)}}{n-1}
    \geq
    \pi_N(i)-\alpha
    \geq
    \frac{\pi_N(i)}{2}
\]
for all \(i\), provided \(\alpha\leq \pi_{\min,N}/2\).

Now write
\[
    \widehat P_n(i,j)-P_N(i,j)
    =
    \frac{
        C_{ij}^{(n)}/(n-1)
    }{
        C_i^{(n)}/(n-1)
    }
    -
    \frac{
        \pi_N(i)P_N(i,j)
    }{
        \pi_N(i)
    }.
\]
Let
\[
    a_{ij}=\frac{C_{ij}^{(n)}}{n-1},
    \qquad
    b_i=\frac{C_i^{(n)}}{n-1},
\]
and
\[
    a_{ij}^0=\pi_N(i)P_N(i,j),
    \qquad
    b_i^0=\pi_N(i).
\]
Then
\[
    \widehat P_n(i,j)-P_N(i,j)
    =
    \frac{a_{ij}}{b_i}
    -
    \frac{a_{ij}^0}{b_i^0}.
\]
On the concentration event, \(|a_{ij}-a_{ij}^0|\leq \alpha\) and \(|b_i-b_i^0|\leq \alpha\). 
If \(\alpha\leq \pi_{\min,N}/2\), then \(b_i\geq b_i^0/2\). Hence
\[
\begin{aligned}
\left|
\frac{a_{ij}}{b_i}
-
\frac{a_{ij}^0}{b_i^0}
\right|
&\leq
\left|
\frac{a_{ij}-a_{ij}^0}{b_i}
\right|
+
|a_{ij}^0|
\left|
\frac{1}{b_i}-\frac{1}{b_i^0}
\right|  \\
&\leq
\frac{2\alpha}{\pi_N(i)}
+
\pi_N(i)P_N(i,j)
\frac{|b_i-b_i^0|}{b_i b_i^0} \\
&\leq
\frac{2\alpha}{\pi_{\min,N}}
+
\frac{2\alpha}{\pi_{\min,N}}
=
\frac{4\alpha}{\pi_{\min,N}}.
\end{aligned}
\]
Choosing \(\alpha=\pi_{\min,N}\varepsilon/4\) gives
\[
    |\widehat P_n(i,j)-P_N(i,j)|\leq \varepsilon .
\]
Taking the maximum over \(i,j\) gives
\[
    \|\widehat P_n-P_N\|_{\max}\leq \varepsilon
\]
on an event whose complement has probability at most
\[
    C_0N^2\exp\left(-c_0 n\pi_{\min,N}^2\varepsilon^2\right).
\]

Finally,
\[
    \widehat P_{n,\rho}-P_{N,\rho}
    =
    (1-\rho)(\widehat P_n-P_N),
\]
because the regularization matrix \(J_N\) cancels.
Therefore,
\[
    \|\widehat P_{n,\rho}-P_{N,\rho}\|_{\max}
    =
    (1-\rho)\|\widehat P_n-P_N\|_{\max}.
\]
This proves the second statement.
\end{proof}
\paragraph{Important modelling point.}
The matrix \(P_N\) below is the ideal finite-state transition matrix of the binned process
under the observed stationary regime. It should not be identified with the classical Ulam
matrix unless that identification is proved separately for the sampling scheme at hand.

\begin{assumption}[Partition approximation error]
Let \(h_\delta\) denote the invariant density of the observed noisy regime. Let
\[
\Pi_N h_\delta
=
\sum_{i=1}^N
\frac{1}{m(I_i)}
\left(\int_{I_i} h_\delta(x)\,dx\right)\mathbf 1_{I_i}(x)
\]
be the conditional expectation of \(h_\delta\) on the partition \(\mathcal A_N\).
Assume that there exist constants \(C_A>0\) and \(\beta>0\) such that
\[
\|\Pi_N h_\delta-h_\delta\|_{L^1}\le C_A\Delta_N^\beta.
\]
Equivalently, if \(\pi_N(i)=\int_{I_i}h_\delta(x)\,dx\) and \(h_N=E_N\pi_N\), then
\[
\|h_N-h_\delta\|_{L^1}\le C_A\Delta_N^\beta.
\]
\end{assumption}

\begin{remark}
This is a partition approximation assumption for the observed invariant density. It should
not be described as a general Ulam convergence theorem unless \(P_N\) is defined as an
actual Ulam discretisation of the transfer operator and the hypotheses of an Ulam
convergence theorem are verified.
\end{remark}

\begin{proposition}
Let
\[
P=(1-\rho)\widetilde P+\rho J_N,
\qquad
Q=(1-\rho)\widetilde Q+\rho J_N,
\qquad 0<\rho<1,
\]
where \(\widetilde P\) and \(\widetilde Q\) are row-stochastic and
\(J_N=\mathbf 1\mathbf 1^\top/N\).
Let \(\pi_P\) and \(\pi_Q\) denote the unique stationary distributions of \(P\) and \(Q\).
Then
\[
\|\pi_Q-\pi_P\|_1
\le
\frac{1}{\rho}\|Q-P\|_{\mathrm{row},1}
\le
\frac{N}{\rho}\|Q-P\|_{\max},
\]
where
\[
\|A\|_{\mathrm{row},1}
=
\max_{1\le i\le N}
\sum_{j=1}^N |A_{ij}|.
\]
\end{proposition}

\begin{proof}
Let
\[
x:=\pi_Q-\pi_P,
\qquad
e:=\pi_Q(Q-P).
\]
Since \(\pi_Q Q=\pi_Q\) and \(\pi_P P=\pi_P\), we have
\[
x
=
\pi_Q-\pi_P
=
\pi_Q Q-\pi_P P
=
\pi_Q(Q-P)+(\pi_Q-\pi_P)P
=
e+xP.
\]

Since both \(\pi_Q\) and \(\pi_P\) are probability vectors,
\[
x\mathbf 1=0.
\]
Moreover, since \(Q\) and \(P\) are row-stochastic,
\[
e\mathbf 1
=
\pi_Q(Q-P)\mathbf 1
=
\pi_Q(Q\mathbf 1-P\mathbf 1)
=
0.
\]

Now write
\[
P=(1-\rho)\widetilde P+\rho J_N.
\]
For every signed row vector \(y\) satisfying \(y\mathbf 1=0\), we have
\[
yJ_N=0.
\]
Therefore,
\[
yP=(1-\rho)y\widetilde P.
\]
Since \(\widetilde P\) is row-stochastic, it is a contraction in the \(\ell^1\)-norm for row vectors:
\[
\|y\widetilde P\|_1\le \|y\|_1.
\]
Hence, for every signed row vector \(y\) with \(y\mathbf 1=0\),
\[
\|yP\|_1
\le
(1-\rho)\|y\|_1.
\]

Iterating the identity \(x=e+xP\) gives, for every \(m\ge1\),
\[
x
=
e\sum_{k=0}^{m-1}P^k
+
xP^m.
\]
Since \(e\mathbf 1=0\) and \(x\mathbf 1=0\), the contraction estimate gives
\[
\|eP^k\|_1\le (1-\rho)^k\|e\|_1,
\qquad
\|xP^m\|_1\le (1-\rho)^m\|x\|_1.
\]
Thus
\[
\|x\|_1
\le
\sum_{k=0}^{m-1}(1-\rho)^k\|e\|_1
+
(1-\rho)^m\|x\|_1.
\]
Letting \(m\to\infty\), we obtain
\[
\|x\|_1
\le
\frac{1}{\rho}\|e\|_1.
\]

Finally,
\[
\|e\|_1
=
\|\pi_Q(Q-P)\|_1
\le
\|Q-P\|_{\mathrm{row},1}
\le
N\|Q-P\|_{\max}.
\]
Therefore,
\[
\|\pi_Q-\pi_P\|_1
\le
\frac{1}{\rho}\|Q-P\|_{\mathrm{row},1}
\le
\frac{N}{\rho}\|Q-P\|_{\max}.
\]
\end{proof}

\begin{theorem}\label{thm:finite_sample_density}
Suppose the concentration assumption of Section~3 holds, together with positive cell mass,
the partition approximation assumption above, and the regularisation-bias assumption.
Then, for every \(0<\varepsilon<1\),
\[
\mathbb P\!\left(
\|E_N\widehat\pi_{n,\rho}-h_\delta\|_{L^1}
>
\frac{N(1-\rho)}{\rho}\,\varepsilon
+\zeta_N(\rho)
+C_A\Delta_N^\beta
\right)
\le
C_0N^2\exp\!\left(-c_0 n\pi_{\min,N}^2\varepsilon^2\right).
\]
\end{theorem}

\begin{proof}
By the triangle inequality,
\[
\|E_N\widehat\pi_{n,\rho}-h_\delta\|_{L^1}
\le
\|E_N\widehat\pi_{n,\rho}-E_N\pi_{N,\rho}\|_{L^1}
+\|E_N\pi_{N,\rho}-E_N\pi_N\|_{L^1}
+\|E_N\pi_N-h_\delta\|_{L^1}.
\]
For probability vectors \(u,v\in\mathbb R^N\),
\[
\|E_Nu-E_Nv\|_{L^1}=\|u-v\|_1.
\]
By the empirical transition concentration lemma,
\[
\| \widehat P_{n,\rho}-P_{N,\rho}\|_{\max}
\le (1-\rho)\varepsilon
\]
with probability at least
\[
1-C_0N^2\exp\!\left(-c_0 n\pi_{\min,N}^2\varepsilon^2\right).
\]
On this event, the perturbation proposition yields
\[
\|\widehat\pi_{n,\rho}-\pi_{N,\rho}\|_1
\le
\frac{N(1-\rho)}{\rho}\,\varepsilon.
\]
The second term is \(\zeta_N(\rho)\), and the third is bounded by
\(C_A\Delta_N^\beta\).
\end{proof}

\begin{corollary}
If, in addition, the noisy and noiseless invariant densities satisfy
\[
\|h_\delta-h\|_{L^1}\le \eta(\delta),
\]
then
\[
\mathbb P\!\left(
\|E_N\widehat\pi_{n,\rho}-h\|_{L^1}
>
\frac{N(1-\rho)}{\rho}\,\varepsilon
+\zeta_N(\rho)
+C_A\Delta_N^\beta
+\eta(\delta)
\right)
\le
C_0N^2\exp\!\left(-c_0 n\pi_{\min,N}^2\varepsilon^2\right).
\]
\end{corollary}

\section{Change Detection via Empirical Invariant Distributions}
\label{sec:changedetection}

We now use the finite-sample approximation result from Section~3 to construct a
change-detection rule.

\paragraph{Uniformity across regimes.}
Throughout this section, \(h_0\) and \(h_t\) denote the invariant densities of the
\emph{observed} baseline and window regimes, respectively. We either assume that all
candidate regimes belong to a class with common constants
\[
C_0,\ c_0,\ \pi_{\ast,N},\ C_A,\ \beta,
\]
or else index the error radii by regime. For notational simplicity, we use the common-constant
version below.

Suppose that an initial baseline segment
\[
    X_0,X_1,\ldots,X_{n_0-1}
\]
is generated under a reference regime with invariant density \(h_0\).
From this segment, we construct the regularized empirical transition matrix
\[
    \widehat P_{0,\rho}
\]
and compute its unique stationary distribution
\[
    \widehat \pi_{0,\rho}.
\]

For a later sliding window of length \(n_t\), starting at time \(t\), we use the observations
\[
    X_t,X_{t+1},\ldots,X_{t+n_t-1}.
\]
From this window, we construct the regularized empirical transition matrix
\[
    \widehat P_{t,\rho}
\]
and compute its unique stationary distribution
\[
    \widehat \pi_{t,\rho}.
\]

The empirical anomaly score is defined by
\[
    S_t
    =
    \|\widehat \pi_{t,\rho}-\widehat \pi_{0,\rho}\|_1.
\]
Equivalently, using the embedding operator \(\mathcal E_N\),
\[
    S_t
    =
    \|\mathcal E_N\widehat \pi_{t,\rho}
    -
    \mathcal E_N\widehat \pi_{0,\rho}\|_{L^1}.
\]
The two expressions are equal because the embedding \(\mathcal E_N\) preserves the \(L^1\)-distance between probability vectors on the partition.

We consider the hypotheses
\[
    H_0: h_t=h_0
    \qquad \text{(no change)}
\]
and
\[
    H_1: h_t\neq h_0
    \qquad \text{(change)}.
\]
Here \(h_t\) denotes the true invariant density associated with the regime generating the sliding window at time \(t\).
For comparison with the noiseless invariant density, define
\[
r_{n,N,\rho,\delta}(\varepsilon)
=
\frac{N(1-\rho)}{\rho}\varepsilon
+
\zeta_N(\rho)
+
C_A\Delta_N^\beta
+
\eta(\delta).
\]
For comparison with the observed noisy invariant density \(h_\delta\), the term
\(\eta(\delta)\) is omitted.
For notational convenience, define the baseline error radius
\[
    r_0(\varepsilon)
    =
    r_{n_0,N,\rho,\delta_0}(\varepsilon),
\]
and the window error radius
\[
    r_t(\varepsilon)
    =
    r_{n_t,N,\rho,\delta_t}(\varepsilon),
\]
where \(n_0\) and \(n_t\) are the baseline and window sample sizes, respectively.

Thus, by Theorem~\ref{thm:finite_sample_density},
\[
    \mathbb P
    \left(
        \|\mathcal E_N\widehat\pi_{0,\rho}-h_0\|_{L^1}
        >
        r_0(\varepsilon)
    \right)
    \leq
    C_0N^2
    \exp\left(
        -c_0 n_0\pi_{*,N}^2\varepsilon^2
    \right),
\]
and
\[
    \mathbb P
    \left(
        \|\mathcal E_N\widehat\pi_{t,\rho}-h_t\|_{L^1}
        >
        r_t(\varepsilon)
    \right)
    \leq
    C_0N^2
    \exp\left(
        -c_0 n_t\pi_{*,N}^2\varepsilon^2
    \right).
\]

\subsection{False-alarm control}

We first prove that the anomaly score remains small under the no-change hypothesis.

\begin{theorem}
\label{thm:false_alarm}
Assume \(H_0\), so that \(h_t=h_0\).
Let
\[
    \tau_t(\varepsilon)
    =
    r_0(\varepsilon)+r_t(\varepsilon).
\]
Then
\[
\mathbb P
\left(
    S_t>\tau_t(\varepsilon)
\right)
\leq
C_0N^2
\exp\left(
    -c_0 n_0\pi_{*,N}^2\varepsilon^2
\right)
+
C_0N^2
\exp\left(
    -c_0 n_t\pi_{*,N}^2\varepsilon^2
\right).
\]
\end{theorem}

\begin{proof}
Under \(H_0\), we have \(h_t=h_0\).
Using the triangle inequality,
\[
\begin{aligned}
S_t
&=
\|\mathcal E_N\widehat \pi_{t,\rho}
-
\mathcal E_N\widehat \pi_{0,\rho}\|_{L^1} \\
&\leq
\|\mathcal E_N\widehat \pi_{t,\rho}-h_t\|_{L^1}
+
\|h_t-h_0\|_{L^1}
+
\|h_0-\mathcal E_N\widehat \pi_{0,\rho}\|_{L^1}.
\end{aligned}
\]
Since \(h_t=h_0\), the middle term is zero. Therefore,
\[
S_t
\leq
\|\mathcal E_N\widehat \pi_{t,\rho}-h_t\|_{L^1}
+
\|\mathcal E_N\widehat \pi_{0,\rho}-h_0\|_{L^1}.
\]

Define the events
\[
    A_0
    =
    \left\{
    \|\mathcal E_N\widehat \pi_{0,\rho}-h_0\|_{L^1}
    \leq
    r_0(\varepsilon)
    \right\},
\]
and
\[
    A_t
    =
    \left\{
    \|\mathcal E_N\widehat \pi_{t,\rho}-h_t\|_{L^1}
    \leq
    r_t(\varepsilon)
    \right\}.
\]
On the event \(A_0\cap A_t\), we have
\[
    S_t
    \leq
    r_0(\varepsilon)+r_t(\varepsilon)
    =
    \tau_t(\varepsilon).
\]
Hence
\[
    \{S_t>\tau_t(\varepsilon)\}
    \subseteq
    A_0^c\cup A_t^c.
\]
By the union bound,
\[
    \mathbb P(S_t>\tau_t(\varepsilon))
    \leq
    \mathbb P(A_0^c)+\mathbb P(A_t^c).
\]
Applying Theorem~\ref{thm:finite_sample_density} to the baseline segment and the sliding window gives the stated bound.
\end{proof}
\begin{corollary}
Let \(\mathcal T\) be a finite set of window starting times. Under the no-change hypothesis for
all \(t\in\mathcal T\),
\[
\mathbb P\!\left(\exists\, t\in\mathcal T:\ S_t>\tau_t(\varepsilon)\right)
\le
\sum_{t\in\mathcal T}
\Bigl[
C_0N^2e^{-c_0n_0\pi_{\ast,N}^2\varepsilon^2}
+
C_0N^2e^{-c_0n_t\pi_{\ast,N}^2\varepsilon^2}
\Bigr].
\]
If all windows have the same length \(n_w\), then
\[
\mathbb P\!\left(\exists\, t\in\mathcal T:\ S_t>\tau_t(\varepsilon)\right)
\le
|\mathcal T|\,C_0N^2
\left(
e^{-c_0n_0\pi_{\ast,N}^2\varepsilon^2}
+
e^{-c_0n_w\pi_{\ast,N}^2\varepsilon^2}
\right).
\]
\end{corollary}
\begin{proof}
For each fixed \(t\in\mathcal T\), Theorem~4.1 gives
\[
\mathbb P\!\left(S_t>\tau_t(\varepsilon)\right)
\le
C_0N^2e^{-c_0n_0\pi_{\ast,N}^2\varepsilon^2}
+
C_0N^2e^{-c_0n_t\pi_{\ast,N}^2\varepsilon^2}.
\]
Moreover,
\[
\{\exists\,t\in\mathcal T:\ S_t>\tau_t(\varepsilon)\}
\subseteq
\bigcup_{t\in\mathcal T}
\{S_t>\tau_t(\varepsilon)\}.
\]
Applying the union bound gives the first inequality. If \(n_t=n_w\) for every
\(t\in\mathcal T\), the second inequality follows immediately.
\end{proof}
\subsection{Detection under a true change}

We now show that if the true invariant density in the sliding window differs sufficiently from the baseline density, then the score exceeds the threshold with high probability.

\begin{theorem}
\label{thm:detection}
Assume that
\[
    \|h_t-h_0\|_{L^1}
    >
    2\left(r_0(\varepsilon)+r_t(\varepsilon)\right).
\]
Let
\[
    \tau_t(\varepsilon)
    =
    r_0(\varepsilon)+r_t(\varepsilon).
\]
Then
\[
\mathbb P
\left(
    S_t>\tau_t(\varepsilon)
\right)
\geq
1
-
C_0N^2
\exp\left(
    -c_0 n_0\pi_{*,N}^2\varepsilon^2
\right)
-
C_0N^2
\exp\left(
    -c_0 n_t\pi_{*,N}^2\varepsilon^2
\right).
\]
\end{theorem}

\begin{proof}
By the reverse triangle inequality,
\[
\begin{aligned}
S_t
&=
\|\mathcal E_N\widehat \pi_{t,\rho}
-
\mathcal E_N\widehat \pi_{0,\rho}\|_{L^1} \\
&\geq
\|h_t-h_0\|_{L^1}
-
\|\mathcal E_N\widehat \pi_{t,\rho}-h_t\|_{L^1}
-
\|\mathcal E_N\widehat \pi_{0,\rho}-h_0\|_{L^1}.
\end{aligned}
\]

On the event \(A_0\cap A_t\), defined as in the proof of Theorem~\ref{thm:false_alarm}, we have
\[
S_t
\geq
\|h_t-h_0\|_{L^1}
-
r_t(\varepsilon)
-
r_0(\varepsilon).
\]
By assumption,
\[
\|h_t-h_0\|_{L^1}
>
2\left(r_0(\varepsilon)+r_t(\varepsilon)\right).
\]
Therefore, on \(A_0\cap A_t\),
\[
S_t
>
r_0(\varepsilon)+r_t(\varepsilon)
=
\tau_t(\varepsilon).
\]
Thus
\[
A_0\cap A_t
\subseteq
\{S_t>\tau_t(\varepsilon)\}.
\]
Consequently,
\[
\mathbb P(S_t>\tau_t(\varepsilon))
\geq
\mathbb P(A_0\cap A_t)
\geq
1-\mathbb P(A_0^c)-\mathbb P(A_t^c).
\]
Using Theorem~\ref{thm:finite_sample_density} for the two error probabilities gives the result.
\end{proof}

\subsection{Asymptotic consistency}

We next record a consistency statement. 
The result says that, if the sample size increases, the partition is refined, the regularization is reduced, and the noise bias vanishes in a controlled way, then false alarms vanish under \(H_0\), while sufficiently separated alternatives are detected with probability tending to one.

\begin{theorem}
\label{thm:consistency}
Let \(n_0,n_t\to\infty\), \(N\to\infty\), \(\rho\to0\), and \(\delta\to0\).
Assume that for a sequence \(\varepsilon_n\downarrow0\),
\[
    N^2
    \exp\left(
        -c_0 n_0\pi_{*,N}^2\varepsilon_n^2
    \right)
    \to0
\]
and
\[
    N^2
    \exp\left(
        -c_0 n_t\pi_{*,N}^2\varepsilon_n^2
    \right)
    \to0.
\]
Assume also that
\[
    r_0(\varepsilon_n)\to0,
    \qquad
    r_t(\varepsilon_n)\to0.
\]
Set
\[
    \tau_t(\varepsilon_n)
    =
    r_0(\varepsilon_n)+r_t(\varepsilon_n).
\]

Then under \(H_0\),
\[
    \mathbb P(S_t>\tau_t(\varepsilon_n))\to0.
\]

Moreover, under \(H_1\), if there exists \(\gamma>0\) such that
\[
    \|h_t-h_0\|_{L^1}\geq \gamma
\]
for all sufficiently large \(n\), then
\[
    \mathbb P(S_t>\tau_t(\varepsilon_n))\to1.
\]
\end{theorem}

\begin{proof}
The false-alarm statement follows directly from Theorem~\ref{thm:false_alarm}. 
The assumed exponential conditions imply that the upper bound for
\[
    \mathbb P(S_t>\tau_t(\varepsilon_n))
\]
goes to zero under \(H_0\).

For the detection statement, since
\[
    r_0(\varepsilon_n)+r_t(\varepsilon_n)\to0,
\]
there exists \(n\) sufficiently large such that
\[
    2\left(r_0(\varepsilon_n)+r_t(\varepsilon_n)\right)
    <
    \gamma.
\]
Since \(\|h_t-h_0\|_{L^1}\geq \gamma\), the condition of Theorem~\ref{thm:detection} holds for all sufficiently large \(n\). 
The probability lower bound in Theorem~\ref{thm:detection} tends to one by the exponential assumptions. 
Therefore,
\[
    \mathbb P(S_t>\tau_t(\varepsilon_n))\to1.
\]
\end{proof}

\subsection{Practical threshold calibration}

The theoretical threshold
\[
    \tau_t(\varepsilon)=r_0(\varepsilon)+r_t(\varepsilon)
\]
depends on constants such as \(C_0,c_0,\pi_{\ast,N},C_A,\beta,\zeta_N(\rho)\),
and the noise modulus \(\eta(\delta)\). 
These constants are generally unknown in applications.

Therefore, in numerical experiments we distinguish between the theoretical threshold and an empirical threshold. 
In practice, the calibration windows should be taken from a period believed to be governed by the baseline regime. 
If anomalous windows are included in the calibration set, the empirical threshold may be inflated and the detector may become less sensitive.
A practical threshold is
\[
    \tau_{\mathrm{emp}}
    =
    \operatorname{Quantile}_{1-\alpha}
    \{S_t:\text{ calibration windows from the baseline regime}\}.
\]
For example, choosing the \(90\%\) or \(95\%\) quantile gives a data-driven threshold.

The theoretical results justify the form of the score and show that, under suitable assumptions, scores remain small under no change. 
The empirical threshold provides a practical way to implement the method when the constants in the finite-sample bound are unknown.
\section{Operator-Theoretic Interpretation}
\label{sec:operator}

The previous sections analyze change detection through empirical stationary distributions obtained from finite transition matrices. 
We now explain how this finite-dimensional procedure relates to perturbations of transfer operators.

Let \(P_0\) and \(P_t\) denote the transfer operators associated with the baseline regime and the regime generating the sliding window at time \(t\), respectively. 
Let \(h_0\) and \(h_t\) be their invariant densities:
\[
    P_0 h_0 = h_0,
    \qquad
    P_t h_t = h_t.
\]
The anomaly score introduced above is designed to estimate
\[
    \|h_t-h_0\|_{L^1}.
\]
Thus, the method detects changes in the invariant density. Such changes may be induced
by changes in the underlying transfer operator, but a transfer-operator change need not
change the invariant density.

\subsection{Perturbation of invariant densities}

The following result gives a standard sufficient condition under which a small perturbation of the transfer operator leads to a small change in the invariant density.

\begin{theorem}
\label{thm:operator_perturbation}
[Operator perturbation principle under Keller--Liverani hypotheses]
Let \(B\) be a Banach space continuously embedded in \(L^1(I)\), and assume that the closed
unit ball of \(B\) is relatively compact in \(L^1(I)\), or that an equivalent quasi-compactness
hypothesis holds. Let \(P_0:B\to B\) and \(P_t:B\to B\) be Markov operators satisfying a
uniform Lasota--Yorke inequality
\[
\|P_k^n f\|_B \le A\alpha^n\|f\|_B + C\|f\|_{L^1},
\qquad k\in\{0,t\},
\]
for some \(A,C>0\) and \(\alpha\in(0,1)\). Assume further that
\[
\|P_t-P_0\|_{B\to L^1}\to 0
\]
and that \(1\) is a simple isolated eigenvalue of \(P_0\), corresponding to a unique
invariant density \(h_0\in B\).

Then, for \(t\) sufficiently small, \(P_t\) has a unique invariant density \(h_t\in B\) near
\(h_0\), and
\[
\|h_t-h_0\|_{L^1}\to 0
\qquad\text{as}\qquad
\|P_t-P_0\|_{B\to L^1}\to 0.
\]
More quantitative moduli of continuity are available under the full Keller--Liverani
hypotheses.
\end{theorem}

\begin{proof}
This is an application of the Keller--Liverani perturbation theorem in the standard
two-norm setting. The uniform Lasota--Yorke inequality and the compactness/quasi-compactness
hypothesis give control of the essential spectral radius, while the small mixed-norm
perturbation preserves the isolated spectral data near the eigenvalue \(1\).
\end{proof}

\begin{remark}
Theorem~\ref{thm:operator_perturbation} is not a new perturbation theorem. 
It records the operator-theoretic mechanism behind the proposed detection statistic. 
The novelty of the present work is the use of empirical finite-dimensional approximations of these invariant densities for change detection with finite-sample control.
\end{remark}

\subsection{ Connection with the empirical score}

The empirical score
\[
    S_t
    =
    \|\widehat\pi_{t,\rho}-\widehat\pi_{0,\rho}\|_1
\]
is a finite-dimensional approximation of
\[
    \|h_t-h_0\|_{L^1}.
\]
The finite-sample results in Section~\ref{sec:finite_sample} quantify the error between the empirical stationary distributions and the corresponding true invariant densities. 
The perturbation result above explains why a change in the transfer operator can lead to a detectable change in invariant density.

The empirical score should therefore be interpreted as follows:

The empirical score
\[
S_t=\|\widehat\pi_{t,\rho}-\widehat\pi_{0,\rho}\|_1
\]
is a finite-dimensional approximation of \(\|h_t-h_0\|_{L^1}\). Therefore the method
detects those changes in dynamics that alter the invariant density. A change in the
transfer operator may, but need not, change the invariant density.
\section{Algorithmic Implementation}
\label{sec:algorithm}

We now describe the practical implementation of the proposed change-detection method. 
The algorithm takes as input a one-dimensional time series and returns anomaly scores for sliding windows.

\subsection{Input}

Let
\[
    X_{\min}=\min_{0\leq t<M}X_t,
    \qquad
    X_{\max}=\max_{0\leq t<M}X_t.
\]
If \(X_{\max}>X_{\min}\), we normalize by
\[
    \widetilde X_t
    =
    \frac{X_t-X_{\min}}{X_{\max}-X_{\min}}.
\]
If \(X_{\max}=X_{\min}\), the time series is constant on the observed segment and no state-space normalization is needed. 
In that degenerate case, the anomaly score is identically zero unless later observations leave the constant regime.

After normalization, we write \(X_t\in[0,1]\) for simplicity.
\begin{remark}
The statistic \(S_t=\|\widehat\pi_{t,\rho}-\widehat\pi_{0,\rho}\|_1\) compares stationary
distributions. It is therefore designed to detect changes in stationary behaviour. It may
fail to detect a change in transition dynamics if the baseline and window regimes have the
same invariant distribution.

If one wants to detect changes in the transition mechanism itself, one should additionally
monitor an operator-distance statistic such as
\[
S_t^{P}=\|\widehat P_{t,\rho}-\widehat P_{0,\rho}\|_{F}
\qquad\text{or}\qquad
S_t^{\max}=\|\widehat P_{t,\rho}-\widehat P_{0,\rho}\|_{\max}.
\]
A complete transition-change detector would require a separate finite-sample analysis for
this matrix-valued statistic.
\end{remark}

\subsection{Partition of the state space}

Choose a positive integer \(N\) and partition the interval \([0,1]\) into \(N\) intervals:
\[
    I_i=\left[\frac{i-1}{N},\frac{i}{N}\right),
    \qquad i=1,\ldots,N-1,
\]
and
\[
    I_N=\left[\frac{N-1}{N},1\right].
\]
The mesh size is
\[
    \Delta_N=\frac{1}{N}.
\]

Each observation \(X_t\) is assigned to a bin:
\[
    Y_t=i
    \quad\Longleftrightarrow\quad
    X_t\in I_i.
\]

\subsection{Empirical transition matrix}

For a data segment
\[
    X_a,X_{a+1},\ldots,X_b,
\]
with corresponding bin sequence
\[
    Y_a,Y_{a+1},\ldots,Y_b,
\]
we define transition counts
\[
    C_{ij}
    =
    \sum_{t=a}^{b-1}
    \mathbf 1_{\{Y_t=i,\;Y_{t+1}=j\}},
\]
and departure counts
\[
    C_i
    =
    \sum_{t=a}^{b-1}
    \mathbf 1_{\{Y_t=i\}}.
\]
The empirical transition matrix \(\widehat P\) is
\[
    \widehat P(i,j)
    =
    \begin{cases}
        C_{ij}/C_i, & C_i>0, \\[0.8em]
        1/N, & C_i=0.
    \end{cases}
\]
Thus every row of \(\widehat P\) sums to one.

To guarantee a unique stationary distribution, we use the regularized matrix
\[
    \widehat P_\rho
    =
    (1-\rho)\widehat P
    +
    \rho J_N,
    \qquad
    J_N=\frac{\mathbf 1\mathbf 1^\top}{N},
\]
where \(0<\rho<1\).

\subsection{Stationary distribution computation}

The empirical invariant distribution \(\widehat\pi_\rho\) is the stationary distribution of \(\widehat P_\rho\):
\[
    \widehat\pi_\rho^\top \widehat P_\rho
    =
    \widehat\pi_\rho^\top,
    \qquad
    \widehat\pi_\rho\geq0,
    \qquad
    \|\widehat\pi_\rho\|_1=1.
\]
In practice, it can be computed by power iteration.

Starting from the uniform distribution
\[
    \pi^{(0)}=\frac{1}{N}(1,\ldots,1),
\]
iterate
\[
    \pi^{(k+1)}=\pi^{(k)}\widehat P_\rho
\]
until
\[
    \|\pi^{(k+1)}-\pi^{(k)}\|_1
\]
is below a chosen tolerance.

\subsection{Baseline and sliding-window scores}

Choose a baseline length \(n_0\). 
Using the baseline segment
\[
    X_0,\ldots,X_{n_0-1},
\]
construct the regularized empirical transition matrix \(\widehat P_{0,\rho}\) and compute its stationary distribution \(\widehat\pi_{0,\rho}\).

Next, choose a sliding-window length \(n_w\) and a step size \(s\). 
For each window starting at time \(t\), use the segment
\[
    X_t,\ldots,X_{t+n_w-1}
\]
to construct \(\widehat P_{t,\rho}\) and compute \(\widehat\pi_{t,\rho}\).

The anomaly score is
\[
    S_t
    =
    \|\widehat\pi_{t,\rho}-\widehat\pi_{0,\rho}\|_1.
\]
A large value of \(S_t\) indicates that the long-term stationary behavior in the current window differs from the baseline behavior.

\subsection{Thresholding}

The theoretical threshold is
\[
    \tau_t(\varepsilon)
    =
    r_0(\varepsilon)+r_t(\varepsilon),
\]
where the terms are defined in Section~\ref{sec:changedetection}. 
In applications, the constants appearing in this expression are generally unknown. 
Therefore, we also use a data-driven threshold obtained from a calibration set:
\[
    \tau_{\mathrm{emp}}
    =
    \operatorname{Quantile}_{1-\alpha}\{S_t:\text{ calibration windows}\}.
\]
A window is declared anomalous if
\[
    S_t>\tau_{\mathrm{emp}}.
\]

\subsection{Algorithm}

\begin{center}
\begin{minipage}{0.95\textwidth}
\hrule
\vspace{0.5em}
\noindent\textbf{Algorithm 1. Empirical transfer-operator change detection}
\vspace{0.5em}
\hrule
\vspace{0.8em}

\begin{algorithmic}[1]
\Require Time series \(X_0,\ldots,X_{M-1}\), number of bins \(N\), baseline size \(n_0\), window size \(n_w\), step size \(s\), regularization parameter \(\rho\), threshold level \(\alpha\).
\Ensure Anomaly scores \(S_t\) and detected windows.

\State Normalize the data to \([0,1]\), if necessary.
\State Partition \([0,1]\) into \(N\) intervals.
\State Convert observations \(X_t\) to bin labels \(Y_t\).
\State Estimate the baseline transition matrix \(\widehat P_0\) from \(Y_0,\ldots,Y_{n_0-1}\).
\State Regularize \(\widehat P_{0,\rho}=(1-\rho)\widehat P_0+\rho J_N\).
\State Compute the baseline stationary distribution \(\widehat\pi_{0,\rho}\).

\For{\(t=n_0,n_0+s,n_0+2s,\ldots,M-n_w\)}
    \State Estimate the window transition matrix \(\widehat P_t\).
    \State Regularize \(\widehat P_{t,\rho}=(1-\rho)\widehat P_t+\rho J_N\).
    \State Compute \(\widehat\pi_{t,\rho}\).
    \State Compute \(S_t=\|\widehat\pi_{t,\rho}-\widehat\pi_{0,\rho}\|_1\).
\EndFor

\State Compute empirical threshold \(\tau_{\mathrm{emp}}\).
\State Declare window \(t\) anomalous if \(S_t>\tau_{\mathrm{emp}}\).
\State \Return anomaly scores and detected windows.
\end{algorithmic}

\vspace{0.5em}
\hrule
\end{minipage}
\end{center}
\begin{remark}
The score \(S_t=\|\widehat\pi_{t,\rho}-\widehat\pi_{0,\rho}\|_1\) compares empirical
stationary distributions. It is therefore designed to detect changes in stationary behaviour.
It may fail to detect a change in transition dynamics if the baseline and window regimes
have the same invariant distribution. Detecting such transition-mechanism changes would
require an additional statistic, such as
\[
S_t^P=\|\widehat P_{t,\rho}-\widehat P_{0,\rho}\|_F
\quad\text{or}\quad
S_t^{\max}=\|\widehat P_{t,\rho}-\widehat P_{0,\rho}\|_{\max},
\]
together with a separate finite-sample analysis.
\end{remark}
\section{Experimental Validation}\label{sec:experiments}

The purpose of this section is to validate the empirical detector
\[
S_t=\|\widehat\pi_{t,\rho}-\widehat\pi_{0,\rho}\|_1
\]
defined in Sections~\ref{sec:changedetection} and~\ref{sec:algorithm}, where
\(\widehat\pi_{0,\rho}\) is the regularised empirical stationary distribution computed from a
baseline segment and \(\widehat\pi_{t,\rho}\) is the corresponding quantity computed from a
sliding window starting at time \(t\). The experiments are designed to assess detection of
\emph{changes in invariant density or stationary behaviour}. They are not intended to claim
detection of every possible change in transition dynamics. This distinction is consistent with
the operator-theoretic interpretation in Section~\ref{sec:operator} and with perturbation results
of Keller--Liverani type for invariant densities under suitable hypotheses
\cite{keller_liverani_1999}. Since Ulam-type discretisations can exhibit slow convergence rates
in general, we do not fix a single partition size in isolation; instead we include a sensitivity
study over the number of bins \(N\) \cite{bose_murray_2001}.

\subsection{Experimental goals}\label{subsec:exp-goals}

The experimental programme has three goals.

First, we evaluate the empirical false-alarm behaviour of the detector under a genuine
no-change regime. This directly reflects the theoretical role of the threshold in Section~4.

Second, we evaluate whether the detector identifies a genuine change in invariant density in
a synthetic change-point experiment, using a model class that is consistent with the main
theorem class of noisy expanding maps.

Third, we study sensitivity with respect to the numerical parameters of the procedure:
the number of bins \(N\), the baseline length \(n_0\), the sliding-window length \(n_w\), and
the Doeblin regularisation parameter \(\rho\). This is necessary because the theoretical
error decomposition in Section~3 contains a discretisation term, a regularisation term,
and a sampling term, all of which depend on these choices.

\subsection{Synthetic data-generating processes}\label{subsec:exp-dgp}

\paragraph{Core model: noisy beta maps.}
Our main synthetic model is the noisy beta map
\[
X_{k+1} = \bigl(\beta X_k + \delta \xi_k\bigr) \pmod{1},
\qquad
\xi_k \stackrel{\mathrm{i.i.d.}}{\sim} \mathcal N(0,1),
\]
on the state space \([0,1]\). Unless otherwise stated, the default parameters are
\[
\beta_{\mathrm{pre}} = 2.0,
\qquad
\beta_{\mathrm{post}} = 1.7,
\qquad
\delta = 0.02,
\qquad
T = 40000,
\qquad
t_\star = 20000,
\qquad
B = 5000.
\]
Here \(B\) denotes the burn-in length and \(t_\star\) is the true change point.

We consider two regimes.

\begin{enumerate}
\item \textbf{Null regime.}
There is no change:
\[
\beta_{\mathrm{pre}} = \beta_{\mathrm{post}} = 2.0.
\]
This experiment is used to estimate hold-out false-alarm behaviour.

\item \textbf{Alternative regime.}
There is a genuine change:
\[
\beta_{\mathrm{pre}} = 2.0,
\qquad
\beta_{\mathrm{post}} = 1.7.
\]
This experiment is used to estimate detection performance.
\end{enumerate}

\paragraph{Restart and continuation designs.}
Two data-generation designs are relevant.

\begin{enumerate}
\item \textbf{Restart design, default.}
The pre-change and post-change segments are generated separately after burn-in and then
concatenated:
\[
X^{\mathrm{pre}}_0,\ldots,X^{\mathrm{pre}}_{t_\star-1},
\qquad
X^{\mathrm{post}}_0,\ldots,X^{\mathrm{post}}_{T-t_\star-1}.
\]
The observed path is
\[
X_0,\ldots,X_{T-1}
=
X^{\mathrm{pre}}_0,\ldots,X^{\mathrm{pre}}_{t_\star-1},
X^{\mathrm{post}}_0,\ldots,X^{\mathrm{post}}_{T-t_\star-1}.
\]
This design isolates the change in stationary regime and is therefore the default for the
main paper.

\item \textbf{Continuation design, optional robustness check.}
A single trajectory is generated, and at time \(t_\star\) the map parameter is changed while
the current state is retained:
\[
X_{k+1}
=
\begin{cases}
(\beta_{\mathrm{pre}} X_k + \delta \xi_k)\pmod{1}, & k < t_\star,\\[0.4em]
(\beta_{\mathrm{post}} X_k + \delta \xi_k)\pmod{1}, & k \ge t_\star.
\end{cases}
\]
This design is more sensitive to transient carry-over from the pre-change state and should
be reported separately from the default experiment.
\end{enumerate}

For Monte Carlo replication \(r\in\{0,\ldots,R-1\}\), the default seed is \(r\). In the restart
design we use seed \(r\) for the pre-change segment and seed \(r+1\) for the post-change
segment, matching the supplied implementation. In the continuation design a single random
number generator with seed \(r\) is used throughout the whole trajectory.

\paragraph{Optional robustness model: noisy logistic map.}
As an optional out-of-class robustness check, one may replace the post-change regime by the
noisy logistic map
\[
X_{k+1}
=
\bigl(4X_k(1-X_k) + \delta \xi_k\bigr)\pmod{1}.
\]
Because the deterministic logistic map at parameter \(4\) is not uniformly expanding on
\([0,1]\), this experiment should be presented only as a robustness study and not as the main
empirical validation of the expanding-map theory.

\subsection{Implementation details}\label{subsec:exp-implementation}

The detector is implemented exactly as described in Section~\ref{sec:algorithm}. The state
space \([0,1]\) is partitioned into \(N\) equal-width bins
\[
I_i = \left[\frac{i-1}{N},\frac{i}{N}\right),
\qquad i=1,\ldots,N-1,
\qquad
I_N = \left[\frac{N-1}{N},1\right].
\]
Unless otherwise stated, the default choice is
\[
N = 100.
\]

Given a binned trajectory \(Y_0,\ldots,Y_{m-1}\in\{0,\ldots,N-1\}\), the empirical transition
matrix \(\widehat P\) is computed from observed transition counts. If a state is never observed
as a departure state in the relevant segment, its row is replaced by the uniform distribution
on \(\{0,\ldots,N-1\}\). The regularised matrix is then
\[
\widehat P_{\rho}
=
(1-\rho)\widehat P + \rho J_N,
\qquad
J_N = \frac{1}{N}\mathbf 1 \mathbf 1^\top.
\]
The default regularisation level is
\[
\rho = 10^{-4}.
\]
This value matches the change-point code used for the synthetic detection experiments and
provides a numerically stable strictly positive matrix. Since very small \(\rho\) can increase
sensitivity to finite-sample perturbations, we also examine a sensitivity grid over \(\rho\).

The empirical stationary distribution \(\widehat\pi_{\rho}\) is computed by power iteration,
initialised from the uniform distribution,
\[
\pi^{(0)}=\frac{1}{N}(1,\ldots,1),
\qquad
\pi^{(k+1)}=\pi^{(k)}\widehat P_\rho,
\]
until
\[
\|\pi^{(k+1)}-\pi^{(k)}\|_1 < 10^{-13},
\]
with a maximum of \(200000\) iterations. These stopping rules match the supplied script.

The baseline and scanning parameters are
\[
n_0 = 5000,
\qquad
n_w = 1000,
\qquad
s = 100,
\qquad
\alpha = 0.05
\]
by default. The baseline distribution \(\widehat\pi_{0,\rho}\) is computed from
\[
X_0,\ldots,X_{n_0-1}.
\]
The score time series is then computed for window starts
\[
t=n_0,n_0+s,n_0+2s,\ldots,T-n_w,
\]
via
\[
S_t=\|\widehat\pi_{t,\rho}-\widehat\pi_{0,\rho}\|_1.
\]

\paragraph{Calibration and hold-out null evaluation.}
A crucial point is that the windows used to calibrate the empirical threshold must be distinct
from the windows used to estimate false alarms. Let
\[
\mathcal T_{\mathrm{pre}}
=
\{t:\ n_0 \le t \le T-n_w,\ t+n_w<t_\star,\ t\equiv n_0 \pmod s\}
\]
be the set of fully pre-change window starts. Write these starts as
\[
t_1 < t_2 < \cdots < t_m.
\]
Let
\[
g = \left\lceil \frac{n_w}{s} \right\rceil
\]
be a one-window separation gap measured in scan steps, and let
\[
k = \lfloor 0.7m \rfloor.
\]
We define
\[
\mathcal T_{\mathrm{cal}} = \{t_1,\ldots,t_{k-g}\},
\qquad
\mathcal T_{\mathrm{null}} = \{t_{k+1},\ldots,t_m\},
\]
so that calibration and null evaluation are separated by at least one full window length.
For the default setting this gives
\[
|\mathcal T_{\mathrm{cal}}|=88,
\qquad
|\mathcal T_{\mathrm{null}}|=42.
\]
The empirical threshold for replication \(r\) is
\[
\tau_{\mathrm{emp}}^{(r)}
=
\operatorname{Quantile}_{1-\alpha}
\{S_t^{(r)}:\ t\in\mathcal T_{\mathrm{cal}}\}.
\]
False alarms are then evaluated only on \(\mathcal T_{\mathrm{null}}\), not on
\(\mathcal T_{\mathrm{cal}}\).

\paragraph{Windows excluded from evaluation.}
Windows that overlap the true change point,
\[
t<t_\star<t+n_w,
\]
contain both regimes and are excluded from the FAR and TPR summaries. Post-change
evaluation starts at the first fully post-change window:
\[
\mathcal T_{\mathrm{alt}}
=
\{t:\ t\ge t_\star,\ t\equiv n_0 \pmod s,\ t\le T-n_w\}.
\]
For the default setting,
\[
|\mathcal T_{\mathrm{alt}}|=191.
\]

\subsection{Evaluation metrics and Monte Carlo protocol}\label{subsec:exp-metrics}

Unless otherwise stated, the main-text experiments use
\[
R=100
\]
independent Monte Carlo replications.

For each replication \(r\), the hold-out false-alarm rate is
\[
\mathrm{FAR}^{(r)}
=
\frac{1}{|\mathcal T_{\mathrm{null}}|}
\sum_{t\in\mathcal T_{\mathrm{null}}}
\mathbf 1\{S_t^{(r)}>\tau_{\mathrm{emp}}^{(r)}\}.
\]
The family-wise false-alarm indicator is
\[
\mathrm{FWER}^{(r)}
=
\mathbf 1\!\left\{\exists\, t\in\mathcal T_{\mathrm{null}}:\ 
S_t^{(r)}>\tau_{\mathrm{emp}}^{(r)}\right\}.
\]
The true-positive rate is
\[
\mathrm{TPR}^{(r)}
=
\frac{1}{|\mathcal T_{\mathrm{alt}}|}
\sum_{t\in\mathcal T_{\mathrm{alt}}}
\mathbf 1\{S_t^{(r)}>\tau_{\mathrm{emp}}^{(r)}\}.
\]
The first detection delay is
\[
\mathrm{Delay}^{(r)}
=
\min\{t-t_\star:\ t\in\mathcal T_{\mathrm{alt}},\ S_t^{(r)}>\tau_{\mathrm{emp}}^{(r)}\},
\]
with the convention that \(\mathrm{Delay}^{(r)}\) is recorded as missing if no post-change
window is detected in replication \(r\).

For each reported metric \(M\), we present
\[
\overline M \pm \operatorname{sd}(M)
\]
over the \(R\) replications. For delay, the mean and standard deviation are computed over
replications with at least one post-change alarm, and the number of missed detections is
reported separately.

\subsection{Default beta-map results}\label{subsec:exp-default-results}

Table~\ref{tab:default-mc} reports the default-setting Monte Carlo results. Under the null
setting, TPR and detection delay are not meaningful because there is no true change point.
They are therefore omitted.

\begin{table}[htbp]
\centering
\small
\caption{Default-setting Monte Carlo results for the synthetic beta-map experiments.
Values are reported as mean \(\pm\) standard deviation over \(R=100\) replications.
For the null setting, TPR and delay are not applicable.}
\label{tab:default-mc}
\begin{tabular}{lcccc}
\hline
Setting & FAR & FWER & TPR & Delay \\
\hline
Null: \(\beta_{\mathrm{pre}}=\beta_{\mathrm{post}}=2.0\)
&
\(0.0798\pm0.1056\)
&
\(0.6600\pm0.4761\)
&
--
&
--
\\
Alternative: \(\beta_{\mathrm{pre}}=2.0,\ \beta_{\mathrm{post}}=1.7\)
&
\(0.0798\pm0.1056\)
&
\(0.6600\pm0.4761\)
&
\(0.8782\pm0.1463\)
&
\(75.00\pm281.90\)
\\
\hline
\end{tabular}
\end{table}

In the default alternative experiment, the detector achieved a mean true-positive rate of
\[
0.8782\pm0.1463,
\]
with no missed detections over \(100\) Monte Carlo replications. The mean detection delay was
\[
75.00\pm281.90
\]
time steps. This supports the use of \(S_t\) for detecting changes that alter the stationary
distribution of the system.

The hold-out pointwise false-alarm rate was
\[
0.0798\pm0.1056.
\]
The family-wise false-alarm rate over the scanned hold-out windows was
\[
0.6600\pm0.4761.
\]
This larger family-wise value is expected because many overlapping windows are scanned using
a pointwise empirical threshold. Thus, the empirical quantile threshold should be interpreted
as a pointwise threshold, not as a family-wise error-control procedure. This observation is
consistent with the multiple-window union-bound discussion in Section~4.

\subsection{Sensitivity study}\label{subsec:exp-sensitivity}

The default setting is
\[
(N,n_0,n_w,s,\rho,\alpha)
=
(100,5000,1000,100,10^{-4},0.05).
\]
To assess robustness, we perform one-factor-at-a-time sweeps around the default setting:
\begin{align*}
N &\in \{50,100,200\},\\
n_w &\in \{500,1000,2000\},\\
\rho &\in \{10^{-6},10^{-5},10^{-4},10^{-3}\},\\
n_0 &\in \{2500,5000,10000\}.
\end{align*}
In each sweep, the remaining parameters are held fixed at their default values.

\begin{table}[htbp]
\centering
\small
\caption{One-factor-at-a-time sensitivity sweep for the alternative beta-map
change-point experiment. All parameters not shown in the first column are kept at their
default values. Values are reported as mean \(\pm\) standard deviation over \(R=100\)
replications.}
\label{tab:sensitivity-mc}
\begin{tabular}{lcccc}
\hline
Setting & FAR & FWER & TPR & Delay \\
\hline
\(N=50\)
&
\(0.0767\pm0.1046\)
&
\(0.6000\pm0.4924\)
&
\(0.9423\pm0.1126\)
&
\(27.00\pm139.88\)
\\
\(N=100\)
&
\(0.0798\pm0.1056\)
&
\(0.6600\pm0.4761\)
&
\(0.8782\pm0.1463\)
&
\(75.00\pm281.90\)
\\
\(N=200\)
&
\(0.0786\pm0.1031\)
&
\(0.6900\pm0.4648\)
&
\(0.7528\pm0.1930\)
&
\(190.00\pm542.25\)
\\
\hline
\(n_w=500\)
&
\(0.0632\pm0.0696\)
&
\(0.7100\pm0.4560\)
&
\(0.5160\pm0.1739\)
&
\(236.00\pm459.36\)
\\
\(n_w=1000\)
&
\(0.0798\pm0.1056\)
&
\(0.6600\pm0.4761\)
&
\(0.8782\pm0.1463\)
&
\(75.00\pm281.90\)
\\
\(n_w=2000\)
&
\(0.1103\pm0.1808\)
&
\(0.5200\pm0.5021\)
&
\(0.9964\pm0.0131\)
&
\(5.00\pm50.00\)
\\
\hline
\(\rho=10^{-6}\)
&
\(0.0798\pm0.1056\)
&
\(0.6600\pm0.4761\)
&
\(0.8781\pm0.1463\)
&
\(75.00\pm281.90\)
\\
\(\rho=10^{-5}\)
&
\(0.0798\pm0.1056\)
&
\(0.6600\pm0.4761\)
&
\(0.8781\pm0.1463\)
&
\(75.00\pm281.90\)
\\
\(\rho=10^{-4}\)
&
\(0.0798\pm0.1056\)
&
\(0.6600\pm0.4761\)
&
\(0.8782\pm0.1463\)
&
\(75.00\pm281.90\)
\\
\(\rho=10^{-3}\)
&
\(0.0800\pm0.1056\)
&
\(0.6600\pm0.4761\)
&
\(0.8785\pm0.1460\)
&
\(75.00\pm281.90\)
\\
\hline
\(n_0=2500\)
&
\(0.0678\pm0.0895\)
&
\(0.6400\pm0.4824\)
&
\(0.8244\pm0.1847\)
&
\(98.00\pm458.14\)
\\
\(n_0=5000\)
&
\(0.0798\pm0.1056\)
&
\(0.6600\pm0.4761\)
&
\(0.8782\pm0.1463\)
&
\(75.00\pm281.90\)
\\
\(n_0=10000\)
&
\(0.0771\pm0.1222\)
&
\(0.5100\pm0.5024\)
&
\(0.9094\pm0.1314\)
&
\(91.00\pm380.35\)
\\
\hline
\end{tabular}
\end{table}

The sensitivity analysis shows the expected bias--variance tradeoff. Increasing the number
of bins from \(N=50\) to \(N=200\) reduces detection performance: the mean TPR decreases
from \(0.9423\) to \(0.7528\), and the mean delay increases from \(27.00\) to \(190.00\).
This is consistent with increased transition-count sparsity for fixed window length.

Increasing the window length from \(n_w=500\) to \(n_w=2000\) improves detection power:
the mean TPR increases from \(0.5160\) to \(0.9964\). This reflects the reduction in sampling
error from longer windows, although in general longer windows reduce temporal localisation.
In the present restart-design experiment, the first fully post-change windows are sufficiently
separated from the change point that the larger window length does not increase the reported
delay.

The results are nearly unchanged across
\[
\rho\in\{10^{-6},10^{-5},10^{-4},10^{-3}\},
\]
indicating that the detector is not sensitive to the regularisation parameter in this range.
Increasing the baseline length from \(n_0=2500\) to \(n_0=10000\) improves TPR from
\(0.8244\) to \(0.9094\), although the delay variability remains substantial.

\subsection{Representative score plot}\label{subsec:score-plot}

Figure~\ref{fig:score-series} displays a representative score time series from the default
alternative experiment. The plot should show the empirical threshold, the true change point,
and the calibration and hold-out null intervals.

\begin{figure}[htbp]
\centering
\includegraphics[width=0.92\linewidth]{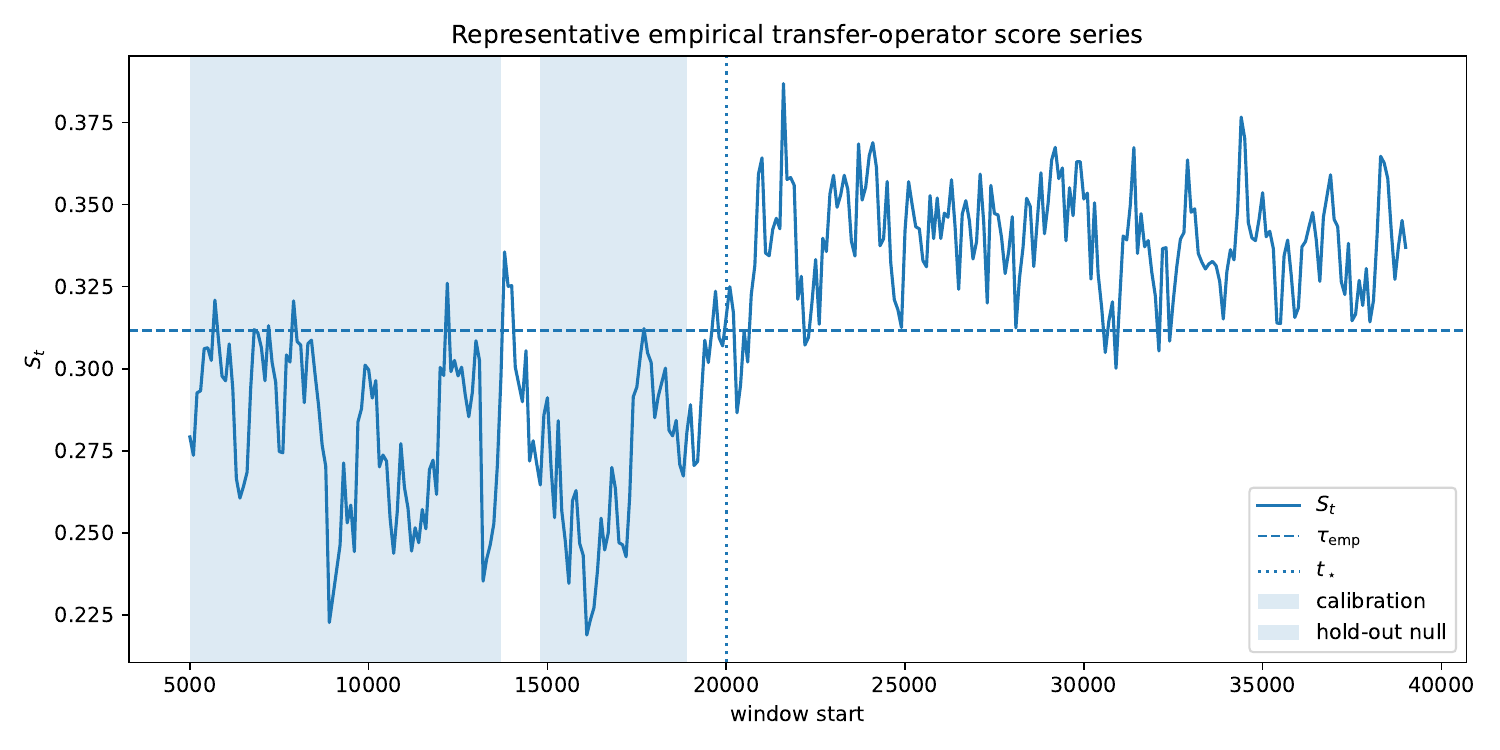}
\caption{Representative score series \(t\mapsto S_t\) for one default alternative
replication of the noisy beta-map change-point experiment. The plot includes the empirical
threshold \(\tau_{\mathrm{emp}}\), the true change point \(t_\star\), and visual indications of
the calibration interval and the hold-out null interval.}
\label{fig:score-series}
\end{figure}

\subsection{Interpretation and limitations}\label{subsec:exp-interpretation}

The correct interpretation of these experiments is the following. A high TPR together with a
moderate pointwise hold-out FAR indicates that the score \(S_t\) is able to detect a change in
\emph{stationary behaviour}. Conversely, weak detection does not by itself imply that the
transition law has not changed. It may instead indicate that the baseline and post-change
regimes have similar invariant densities at the chosen noise level and resolution, or that the
available window length is too small relative to the discretisation.

For this reason, operator-distance summaries such as
\[
\|\widehat P_{t,\rho}-\widehat P_{0,\rho}\|_F
\]
should be treated, if reported at all, as auxiliary diagnostics rather than as the primary
validation of the main detector. The main detector is built from stationary distributions, not
from an operator norm.

The experiments therefore validate the proposed method as a detector of changes in invariant
density or stationary behaviour for noisy one-dimensional maps. They do not establish
detection of all possible transfer-operator perturbations, because two different transition
mechanisms may share the same invariant density.
\section{Conclusion}\label{sec:conclusion}

We have developed a partition-based empirical transfer-operator approach for detecting
changes in the stationary behaviour of one-dimensional noisy dynamical systems. The method
uses observations from a baseline segment to estimate a regularised empirical transition
matrix and its stationary distribution. For each subsequent sliding window, a corresponding
empirical stationary distribution is computed, and the score
\[
S_t=\|\widehat\pi_{t,\rho}-\widehat\pi_{0,\rho}\|_1
\]
is used to measure deviation from the baseline regime.

The main theoretical contribution is a finite-sample approximation bound for the empirical
stationary density associated with the regularised finite-state model. Under assumptions on
finite-state concentration, positive cell mass, partition approximation of the observed
invariant density, regularisation bias, and noise stability, the bound separates the effects of
sampling error, regularisation, partition approximation, and noise. This leads to a
single-window false-alarm guarantee and a sufficient detection condition when the invariant
density of the window regime is separated from that of the baseline regime by more than the
combined estimation error. We also included a multiple-window union-bound corollary to
clarify how false-alarm probabilities accumulate when many windows are scanned.

The numerical experiments on noisy beta-map change-point data support the theoretical
interpretation of the statistic. In the default alternative experiment, the detector achieved a
high true-positive rate with no missed detections over the Monte Carlo replications, while
the sensitivity study showed the expected tradeoff between partition resolution and sampling
variability. The experiments also showed that a pointwise empirical threshold does not
automatically provide family-wise control over many scanned windows, which is consistent
with the multiple-window discussion in the theory.

The scope of the method should be interpreted carefully. The statistic \(S_t\) compares
stationary distributions and is therefore designed to detect changes in invariant density or
stationary behaviour. It does not, by itself, detect all possible changes in transition dynamics.
In particular, two different transition mechanisms may share the same invariant density, in
which case the proposed score may remain small. Detecting such changes would require an
additional operator-level statistic, such as a matrix-norm comparison of empirical transition
matrices, together with a separate finite-sample analysis.

Several extensions remain for future work. One direction is to develop finite-sample theory
for operator-level change statistics that can detect changes in transition structure even when
the invariant density is unchanged. Another direction is to replace the pointwise empirical
threshold by a threshold calibrated for the maximum score over a collection of windows, in
order to obtain stronger family-wise false-alarm control in practical scanning procedures.
Finally, the method should be tested on real time-series data with externally validated
change points, while clearly distinguishing exploratory detections from ground-truth
performance claims.
\appendix
\section{Reproducibility Details}
The complete Python script used to generate Tables~\ref{tab:default-mc}
and~\ref{tab:sensitivity-mc}, as well as Figure~\ref{fig:score-series}, is provided
in the supplementary material.

\end{document}